\documentclass[mathpazo]{cicp}
\usepackage{bm}
\newtheorem{assumption}{Assumption}

\usepackage{graphicx}
\usepackage{tabularx}
\usepackage{mathrsfs}%
\usepackage{enumitem}
\usepackage{cite}
\usepackage{algorithm}%
\usepackage{algorithmicx}%
\usepackage{algpseudocode}
\usepackage[colorlinks=true,linkcolor=blue]{hyperref}
\usepackage{float}
\usepackage{makecell}
\usepackage{amsmath,amsthm, amssymb, graphicx, caption, subcaption}  
\usepackage{amsfonts}
\usepackage{booktabs}  
\usepackage{geometry}
\usepackage{hyperref}
\usepackage{caption}
\newtheorem{condition}[proposition]{Condition}
\newcommand{\norm}[1]{\left\|#1\right\|}
\newcommand{\K}{\mathcal{K}}
\newcommand{\R}{\mathcal{R}}

\begin{document}
	\title{Interpretable Operator Learning for Inverse Problems via Adaptive Spectral Filtering: Convergence and Discretization Invariance}
	

	
\author[Cheng P et.~al.]
{
Hang-Cheng Dong\affil{1},
Pengcheng Cheng\affil{2}\comma\corrauth,
Shuhuan Li\affil{3}
}
\address{\affilnum{1}\ School of Instrumentation Science and Engineering, Harbin Institute of Technology, Harbin, 150001, China.\\
\affilnum{2}\ School of Mathematics, Jilin University, Changchun, 130012, China.\\
\affilnum{3}\ School of Mathematics, Inner Mongolia University, Hohhot, 010021, China.
}
 \emails{{\tt hunsen\_d@hit.edu.cn} (H.~Dong),{\tt chengpc1022@mails.jlu.edu.cn} (P.~Cheng), {\tt 2219380593@qq.com} (S.~Li)}

	
\begin{abstract}
Solving ill-posed inverse problems necessitates effective regularization strategies to stabilize the inversion process against measurement noise. While classical methods like Tikhonov regularization require heuristic parameter tuning, and standard deep learning approaches often lack interpretability and generalization across resolutions, we propose SC-Net (Spectral Correction Network), a novel operator learning framework. SC-Net operates in the spectral domain of the forward operator, learning a pointwise adaptive filter function that reweights spectral coefficients based on the signal-to-noise ratio. We provide a theoretical analysis showing that SC-Net approximates the continuous inverse operator, guaranteeing discretization invariance. Numerical experiments on 1D integral equations demonstrate that SC-Net: (1) achieves the theoretical minimax optimal convergence rate ($O(\delta^{0.5})$ for $s=p=1.5$), matching theoretical lower bounds; (2) learns interpretable sharp-cutoff filters that outperform Oracle Tikhonov regularization; and (3) exhibits zero-shot super-resolution, maintaining stable reconstruction errors ($\approx 0.23$) when trained on coarse grids ($N=256$) and tested on significantly finer grids (up to $N=2048$). The proposed method bridges the gap between rigorous regularization theory and data-driven operator learning.
\end{abstract}
	
\maketitle
\section{Introduction}
\label{sec:intro}

Inverse problems constitute the mathematical backbone of numerous scientific and engineering disciplines, aiming to recover unknown physical quantities from indirect and noisy observations. Such problems are ubiquitous, ranging from medical imaging modalities like Computed Tomography (CT) and Magnetic Resonance Imaging (MRI)~\cite{kak2001principles,lustig2007sparse} to geophysical exploration~\cite{tarantola2005inverse}, electromagnetic scattering~\cite{colton2013inverse}, and optical microscopy~\cite{mertz2019introduction}. Mathematically, these problems are typically modeled as Fredholm integral equations of the first kind,
\begin{equation}
    y = \mathcal{K}f + \epsilon,
\end{equation}
where $\mathcal{K}$ is a compact linear operator, $f$ is the signal of interest, and $\epsilon$ represents inevitable measurement noise. Following Hadamard’s postulates~\cite{hadamard1902problemes}, such problems are notoriously ill-posed: the compactness of $\mathcal{K}$ implies that its singular values decay to zero, causing the inverse mapping to be unbounded. Consequently, naive inversion attempts amplify high-frequency noise, rendering the solution physically meaningless~\cite{engl1996regularization}.

To mitigate numerical instability, regularization theory replaces the ill-posed problem with a well-posed approximation. Classical deterministic methods, such as Tikhonov regularization~\cite{tikhonov1963solution} and Total Variation (TV) minimization~\cite{rudin1992nonlinear}, stabilize the inversion by imposing prior constraints—smoothness or sparsity—on the solution. Spectral methods, including Truncated Singular Value Decomposition (TSVD)~\cite{hansen2010discrete}, explicitly filter out high-frequency components associated with small singular values. While these approaches are grounded in rigorous convergence theory~\cite{natterer1986mathematics}, their practical efficacy is often hampered by the difficulty of selecting optimal regularization parameters. Heuristic techniques like the Discrepancy Principle~\cite{morozov1966solution} or the L-curve method~\cite{hansen1992analysis} are computationally expensive or sensitive to noise estimation errors. Furthermore, the hand-crafted priors employed by classical methods (e.g., $L^2$ or $L^1$ norms) are often too simplistic to capture the complex, non-linear manifold structures of real-world high-dimensional signals~\cite{arridge2019solving}.

In the past decade, Deep Learning (DL) has revolutionized the landscape of computational imaging. Early data-driven approaches utilized Convolutional Neural Networks (CNNs) primarily for post-processing, mapping crude reconstructions to high-quality images~\cite{jin2017deep,kang2017wavelet}. Subsequent innovations, such as unrolled optimization networks (e.g., LISTA~\cite{gregor2010learning}, ADMM-Net~\cite{yang2016admm}, and ISTA-Net~\cite{zhang2018ista}), integrated the physics of the forward operator into the network architecture, achieving state-of-the-art reconstruction performance~\cite{adler2017solving}. Despite these successes, a fundamental limitation of standard CNN-based architectures, including the widely used U-Net~\cite{ronneberger2015u}, is their dependence on a fixed discretization. These models learn mappings between finite-dimensional Euclidean spaces ($\mathbb{R}^{N} \to \mathbb{R}^{N}$), which leads to a lack of physical consistency: a model trained on a coarse grid typically fails or degrades significantly when applied to a finer mesh~\cite{ongie2020deep}. This behavior contradicts the underlying physics, where the fields are continuous variables independent of the measurement grid.

To address the resolution-dependence issue, the paradigm of Operator Learning has recently emerged, aiming to learn mappings between infinite-dimensional function spaces. Seminal works such as the Deep Operator Network (DeepONet)~\cite{lu2021learning} and the Fourier Neural Operator (FNO)~\cite{li2021fourier} have demonstrated remarkable success in solving parametric Partial Differential Equations (PDEs) in a resolution-invariant manner~\cite{kovachki2023neural}. However, applying operator learning frameworks to inverse problems presents unique challenges compared to forward modeling. The inverse map is not continuous, and learning it requires explicitly balancing the bias-variance trade-off. Generic operator learners often function as black boxes, making it difficult to analyze their regularization properties or guarantee convergence as the noise level vanishes~\cite{de2023convergence,lanthaler2022error}. While recent works have explored learning regularization functionals or Bayesian surrogates~\cite{zhu2018bayesian}, there remains a gap in developing interpretable operator learning frameworks that possess both the expressivity of deep neural networks and the theoretical guarantees of spectral regularization.

In this work, we bridge this gap by proposing the Spectral Correction Network (SC-Net), a novel operator learning framework designed specifically for ill-posed inverse problems. Unlike standard image-to-image translation models, SC-Net operates in the spectral domain of the forward operator. It leverages a lightweight neural network to learn a pointwise adaptive filter function, which reweights spectral coefficients based on the signal-to-noise ratio of the input data. This design allows SC-Net to approximate the continuous regularized inverse operator directly. We provide a comprehensive theoretical analysis proving that SC-Net achieves the minimax optimal convergence rate $O(\delta^{\frac{s}{s+p}})$ for Sobolev signals, matching theoretical lower bounds~\cite{natterer1986mathematics}. Empirically, we demonstrate that SC-Net not only outperforms Oracle Tikhonov regularization by learning a sharper spectral cutoff but also exhibits zero-shot super-resolution capabilities: a model trained on a coarse grid ($N=256$) yields stable reconstruction errors on significantly finer grids (up to $N=2048$) without retraining. This work thus offers a theoretically sound and practically robust solution for resolution-invariant inverse problem solving.

The remainder of this paper is organized as follows. Section \ref{2} provides the necessary mathematical background on the spectral theory of compact operators and reviews classical regularization strategies for ill-posed inverse problems. Section \ref{3} details the architecture of our SC-Net framework, describing how the learnable spectral filters are parameterized and optimized to approximate the continuous inverse operator. Section \ref{4} establishes the theoretical foundations of our method, providing rigorous proofs for the minimax optimal convergence rates and discretization invariance. Section \ref{sec:numerical_experiments} presents a comprehensive set of numerical experiments, including empirical verification of convergence orders, visualization of the learned interpretable filters, and zero-shot generalization tests across varying grid resolutions. Finally, Section \ref{sec:conclusion} concludes the paper with a summary of our contributions and a discussion of future research directions.

\section{Problem Formulation and Operator Framework}
\label{2}

In this section, we rigorously formulate the inverse source problem within a functional analysis framework. We introduce the forward operator, analyze its spectral properties through the Singular Value Decomposition (SVD), and delineate the mathematical nature of its ill-posedness using the Picard criterion. Finally, we characterize classical regularization methods as spectral filtering, which motivates our proposed learning-based framework.

\subsection{The Forward Problem}

Let $\Omega \subset \mathbb{R}^d$ ($d=2,3$) be a bounded domain with a Lipschitz continuous boundary $\partial \Omega$. We consider the inverse source problem governing a linear elliptic partial differential equation (PDE). Without loss of generality, we focus on the stationary diffusion-reaction equation (or Helmholtz equation with imaginary wavenumber) with homogeneous Dirichlet boundary conditions:
\begin{equation}
\begin{cases}
-\Delta u(\mathbf{x}) + \kappa u(\mathbf{x}) = f(\mathbf{x}), & \mathbf{x} \in \Omega, \\
u(\mathbf{x}) = 0, & \mathbf{x} \in \partial \Omega,
\end{cases}
\label{eq:pde}
\end{equation}
where $\kappa \ge 0$ is a constant absorption coefficient, $f \in L^2(\Omega)$ denotes the unknown source term, and $u \in H^2(\Omega) \cap H_0^1(\Omega)$ is the state variable.

Let $X = L^2(\Omega)$ and $Y = L^2(\partial \Omega)$ denote the Hilbert spaces for the source and the measurement, respectively, equipped with standard inner products $\langle \cdot, \cdot \rangle_X$ and $\langle \cdot, \cdot \rangle_Y$.

We assume the measurement data consists of the normal derivative of the field on the boundary (Neumann data). The forward operator $\mathcal{K}: X \to Y$ is defined as the composition of the solution operator $\mathcal{S}$ and the trace operator $\mathcal{T}$:
\[
\mathcal{K}f := \mathcal{T}(\mathcal{S}(f)) = \frac{\partial u}{\partial \nu}\bigg|_{\partial \Omega},
\]
where $\nu$ denotes the outward unit normal vector.

\begin{proposition}[Compactness of $\mathcal{K}$]
The operator $\mathcal{K}: L^2(\Omega) \to L^2(\partial \Omega)$ is a linear compact operator.
\end{proposition}

\begin{proof}
From standard elliptic regularity theory, for any $f \in L^2(\Omega)$, the unique weak solution $u$ satisfies $\|u\|_{H^2(\Omega)} \le C \|f\|_{L^2(\Omega)}$. The trace operator $\mathcal{T}_\nu: u \mapsto \frac{\partial u}{\partial \nu}$ is bounded from $H^2(\Omega)$ to $H^{1/2}(\partial \Omega)$. Since the embedding $H^{1/2}(\partial \Omega) \hookrightarrow L^2(\partial \Omega)$ is compact by the Rellich--Kondrachov theorem, the composition $\mathcal{K}$ is a compact operator. 
\end{proof}

\subsection{Spectral Decomposition and Ill-posedness}

Since $\mathcal{K}$ is compact, we can analyze the inverse problem $\mathcal{K}f = y$ using the Singular Value Decomposition (SVD). Let $\{\sigma_n, v_n, u_n\}_{n=1}^\infty$ be the singular system of $\mathcal{K}$, where:
\begin{itemize}[leftmargin=*]
    \item $\sigma_1 \ge \sigma_2 \ge \dots > 0$ are the singular values, which accumulate to zero ($\lim_{n \to \infty} \sigma_n = 0$).
    \item $\{v_n\}_{n=1}^\infty$ is an orthonormal basis of $(\operatorname{Ker}(\mathcal{K}))^\perp \subset X$.
    \item $\{u_n\}_{n=1}^\infty$ is an orthonormal basis of $\overline{\operatorname{Ran}(\mathcal{K})} \subset Y$.
\end{itemize}

The operator $\mathcal{K}$ and its adjoint $\mathcal{K}^*: Y \to X$ admit the spectral representations:
\[
\mathcal{K}f = \sum_{n=1}^\infty \sigma_n \langle f, v_n \rangle_X u_n, \quad \text{and} \quad \mathcal{K}^* y = \sum_{n=1}^\infty \sigma_n \langle y, u_n \rangle_Y v_n.
\]

The inverse source problem seeks to recover $f$ from noisy measurements $y^\delta$ satisfying $\|y^\delta - y\|_{Y} \le \delta$. Formally, the least-squares solution is given by:
\begin{equation}
f = \mathcal{K}^\dagger y = \sum_{n=1}^\infty \frac{1}{\sigma_n} \langle y, u_n \rangle_Y v_n,
\label{eq:naive_inverse}
\end{equation}
where $\mathcal{K}^\dagger$ is the Moore--Penrose pseudoinverse.

\begin{definition}[Ill-posedness]
The problem is ill-posed because $\sigma_n \to 0$ as $n \to \infty$. Consequently, the operator $\mathcal{K}^\dagger$ is unbounded. Small high-frequency perturbations in the data (where $n$ is large) are amplified by the factor $1/\sigma_n$, causing the series in~\eqref{eq:naive_inverse} to diverge.
\end{definition}

\begin{condition}[Discrete Picard Condition]
For the exact data $y \in \operatorname{Ran}(\mathcal{K})$, a square-integrable solution $f \in L^2(\Omega)$ exists if and only if:
\[
\sum_{n=1}^\infty \frac{|\langle y, u_n \rangle_Y|^2}{\sigma_n^2} < \infty.
\]
This condition implies that the Fourier coefficients of the noise-free data must decay faster than the singular values. However, for noisy data $y^\delta$, this condition is almost surely violated, necessitating regularization.
\end{condition}

\subsection{Regularization as Spectral Filtering}

Classical regularization methods can be unified under the framework of spectral filtering. A regularized solution $f_\alpha^\delta$ is constructed by modifying the singular values:
\begin{equation}
f_\alpha^\delta = \mathcal{R}_\alpha(y^\delta) := \sum_{n=1}^\infty g_\alpha(\sigma_n) \langle y^\delta, u_n \rangle_Y v_n,
\label{eq:spectral_filter}
\end{equation}
where $g_\alpha: (0, \|\mathcal{K}\|] \to \mathbb{R}$ is a filter function parameterized by a regularization parameter $\alpha > 0$.

Standard methods correspond to specific choices of $g_\alpha(\sigma)$:
\begin{itemize}[leftmargin=*]
    \item Tikhonov Regularization: $g_\alpha(\sigma) = \dfrac{\sigma}{\sigma^2 + \alpha}$. This corresponds to minimizing $\| \mathcal{K}f - y \|^2 + \alpha \|f\|^2$.
    \item Truncated SVD (TSVD): $g_\alpha(\sigma) = \dfrac{1}{\sigma}$ if $\sigma \ge \sqrt{\alpha}$, and $0$ otherwise.
    \item Landweber Iteration: $g_k(\sigma) = \dfrac{1}{\sigma} \bigl(1 - (1 - \tau \sigma^2)^k\bigr)$, where $k \sim 1/\alpha$ is the iteration number.
\end{itemize}

\paragraph{Limitations of Classical Filters:}
While mathematically elegant, these analytical filters are isotropic and rely solely on the singular values $\sigma_n$. They do not exploit the structure of the spectral coefficients $\langle y, u_n \rangle$ themselves. In complex physical scenarios, the optimal regularization strategy may depend on the signal-to-noise ratio (SNR) distribution across the spectrum, which varies for different classes of source functions.

\paragraph{Motivation for our Approach:}
Instead of prescribing a fixed analytical form for $g_\alpha(\sigma)$, we propose to learn a spectral filter function $\Psi_\theta(\sigma_n, \langle y^\delta, u_n \rangle)$ using a neural network. By operating directly in the spectral domain defined by the singular system of $\mathcal{K}$, our method (SC-Net) preserves the rotational invariance of the operator geometry while allowing for highly non-linear, data-driven adaptivity. This formulation allows us to rigorously analyze the approximation and stability errors in the subsequent sections.

\section{Methodology: Learnable Spectral Filtering}
\label{3}

Motivated by the spectral decomposition analysis in Section~\ref{2}, we propose a novel data-driven regularization framework termed Spectral-Consistent Neural Network (SC-Net). Unlike standard deep learning approaches (e.g., U-Net, FNO) that attempt to learn the inverse mapping directly in the spatial domain, SC-Net operates within the spectral domain defined by the singular system of the forward operator. This design ensures that the learned reconstruction respects the intrinsic physics of the ill-posed problem.

\subsection{The SC-Net Architecture}

We define the reconstruction operator $\mathcal{R}_\theta: Y \to X$, parameterized by $\theta \in \Theta$, as a composition of three operations: spectral projection, non-linear spectral filtering, and spectral synthesis.

\subsubsection{Spectral Projection Layer (Analysis)}

Given the noisy measurement $y^\delta \in Y$, we first project it onto the leading $N$ singular vectors $\{u_n\}_{n=1}^N$ of the forward operator $\mathcal{K}$. This acts as an initial dimension reduction and denoising step. The input to the neural network is the vector of spectral coefficients:
\[
\mathbf{y}_N = [y_1, y_2, \dots, y_N]^\top \in \mathbb{R}^N, \quad \text{where } y_n = \langle y^\delta, u_n \rangle_Y.
\]
Here, $N$ is a hyperparameter chosen such that $\sigma_N$ remains above the machine precision threshold, capturing the essential range of the operator's spectrum.

\subsubsection{Learnable Spectral Filter Block}

The core innovation lies in learning a filter function that maps the noisy coefficients $y_n$ and singular values $\sigma_n$ to regularized source coefficients. We define the filter as a component-wise mapping:
\begin{equation}
\hat{f}_n(\theta) = \Psi_\theta(y_n, \sigma_n) \cdot \frac{y_n}{\sigma_n}, \quad n = 1, \dots, N.
\end{equation}
Here, $\Psi_\theta: \mathbb{R} \times \mathbb{R}_+ \to \mathbb{R}$ is a neural network (e.g., a multi-layer perceptron) that outputs a multiplicative correction factor.  
The term $y_n / \sigma_n$ represents the naive (unregularized) inverse solution component. The network $\Psi_\theta$ acts as a soft-thresholding gate, learning to dampen components where the signal-to-noise ratio is low (typically small $\sigma_n$) while preserving reliable components.

\paragraph{Network Structure:}
Specifically, $\Psi_\theta$ is designed as a pointwise network shared across all $n$:
\[
\Psi_\theta(y_n, \sigma_n) = \sigma\bigl(\mathbf{W}_L \dots \sigma(\mathbf{W}_1 [y_n, \sigma_n]^\top + \mathbf{b}_1) \dots + \mathbf{b}_L\bigr),
\]
where $\sigma(\cdot)$ is a bounded activation function (e.g., Sigmoid or Tanh).

\paragraph{Crucial Constraint:}
To ensure the boundedness of the reconstruction operator (see Section~4), we enforce the output range:
\[
\Psi_\theta(y_n, \sigma_n) \in [0, C_{\Psi}],
\]
where $C_{\Psi} \ge 1$ is a fixed constant. A Sigmoid activation at the output layer naturally satisfies this with $C_{\Psi}=1$.

\subsubsection{Spectral Synthesis Layer (Synthesis)}

The final reconstructed source $f_\theta^\delta$ is obtained by mapping the filtered coefficients back to the source space $X$ using the singular vectors $\{v_n\}_{n=1}^N$:
\begin{equation}
f_\theta^\delta = \mathcal{R}_\theta(y^\delta) := \sum_{n=1}^N \hat{f}_n(\theta) v_n = \sum_{n=1}^N \Psi_\theta(y_n, \sigma_n) \frac{\langle y^\delta, u_n \rangle}{\sigma_n} v_n.
\label{eq:sc_net_def}
\end{equation}
This explicit summation ensures that the output lies in the subspace spanned by the first $N$ singular vectors, inherently smoothing the solution.

\subsection{Loss Function with Sobolev Regularization}

To train the network parameters $\theta$, we utilize a dataset of pairs $\{(f^{(i)}, y^{(i)})\}_{i=1}^M$. The loss function is designed to balance data fidelity and physical smoothness.

We employ a Sobolev-weighted loss function:
\begin{equation}
\mathcal{L}(\theta) = \frac{1}{M} \sum_{i=1}^M \left( \norm{ f_\theta(y^{(i)}) - f^{(i)} }_{L^2(\Omega)}^2 + \gamma \norm{ \nabla f_\theta(y^{(i)}) - \nabla f^{(i)} }_{L^2(\Omega)}^2 \right),
\end{equation}
where $\gamma > 0$ is a weighting parameter.

\paragraph{Remark on Computation:}
Since the basis functions $v_n$ are eigenfunctions of the Laplacian (or related differential operators depending on the domain geometry), the gradient norm can often be computed efficiently in the spectral domain without numerical differentiation. For instance, if $-\Delta v_n = \lambda_n v_n$, then:
\[
\norm{ \nabla f_\theta - \nabla f }_{L^2}^2 = \sum_{n=1}^N \lambda_n \bigl|\hat{f}_n(\theta) - \hat{f}_n^{\mathrm{true}}\bigr|^2 + \sum_{n=N+1}^\infty \lambda_n \bigl|\hat{f}_n^{\mathrm{true}}\bigr|^2.
\]
This spectral computation of the Sobolev norm significantly accelerates training and enforces higher-order regularity on the solution.

\subsection{Advantages over Spatial CNNs}

The proposed SC-Net formulation~\eqref{eq:sc_net_def} offers distinct theoretical advantages over standard Convolutional Neural Networks (CNNs): Discretization Invariance: The network learns a mapping between spectral coefficients, which are independent of the spatial mesh resolution. Once trained, the model can be evaluated on any mesh where the singular vectors $v_n$ can be interpolated.
Global Receptive Field: In inverse problems, a local perturbation in $f$ affects the entire boundary measurement $y$. CNNs require deep stacks to achieve a global receptive field. In contrast, the spectral basis functions $v_n$ and $u_n$ are globally supported, naturally capturing long-range dependencies.
Guaranteed Boundedness: By construction, if the activation function of $\Psi_\theta$ is bounded, the operator $\mathcal{R}_\theta$ is a bounded linear operator for fixed $\theta$ (assuming linearity in $y_n$ for the input of $\Psi$, or bounded non-linearity). This property is pivotal for the stability analysis presented in the next section.

\section{Theoretical Analysis}
\label{4}
In this section, we provide a rigorous convergence analysis of the proposed SC-Net framework. We establish that the learned operator $\R_\theta$ constitutes a stable regularization scheme and derive explicit error bounds with respect to the noise level $\delta$.

Throughout this section, let $X = L^2(\Omega)$ and $Y = L^2(\partial \Omega)$. We assume the singular values of $\K$ satisfy $\sigma_n \asymp n^{-p}$ for some $p > 0$, which is typical for elliptic inverse problems (e.g., $p=1$ for potential problems in 2D). We also assume the true source $f^\dagger$ satisfies a source condition $f^\dagger \in H^s(\Omega)$ for some regularity index $s > 0$.

\subsection{Lipschitz Continuity and Stability}

First, we establish the stability of the reconstruction map with respect to perturbations in the measurement data.

\begin{assumption}[Properties of Filter Network]
Let the neural network component $\Psi_\theta: \mathbb{R} \times \mathbb{R}_+ \to \mathbb{R}$ satisfy:
\begin{enumerate}[label=\textup{(\arabic*)}, leftmargin=*]
    \item Boundedness: $|\Psi_\theta(\eta, \sigma)| \le C_{\Psi}$ for all $\eta \in \mathbb{R}, \sigma > 0$.
    \item Lipschitz Continuity in Input: $|\Psi_\theta(\eta_1, \sigma) - \Psi_\theta(\eta_2, \sigma)| \le L_{\Psi} |\eta_1 - \eta_2|$ for all $\sigma > 0$.
\end{enumerate}
Here, $C_{\Psi}$ and $L_{\Psi}$ are constants determined by the network weights and activation functions (e.g., for a network with Sigmoid output, $C_{\Psi}=1$).
\end{assumption}

\begin{theorem}[Stability of SC-Net]
\label{Stability of SC-Net}
Let $\R_\theta$ be the operator defined in~\eqref{eq:sc_net_def} with a truncation index $N$. For any two measurements $y, y' \in Y$, the following stability estimate holds:
\[
\norm{ \R_\theta(y) - \R_\theta(y') }_X \le \frac{C_{\mathrm{stab}}}{\sigma_N} \norm{ y - y' }_Y,
\]
where 
\[
C_{\mathrm{stab}} = 
\begin{cases}
C_{\Psi} + L_{\Psi} \sup_{y \in Y} \norm{y}_Y, & \text{if the network is non-linear in } y, \\
C_{\Psi}, & \text{if linear}.
\end{cases}
\]
Specifically, if we employ the structure $\hat{f}_n = \Psi_\theta(\sigma_n) \frac{y_n}{\sigma_n}$ where $\Psi$ depends only on $\sigma$, then $C_{\mathrm{stab}} = \sup_{n} |\Psi_\theta(\sigma_n)|$.
\end{theorem}

\begin{proof}
From the definition~\eqref{eq:sc_net_def}, using the orthonormality of $\{v_n\}$, we have:
\[
\begin{aligned}
\norm{ \R_\theta(y) - \R_\theta(y') }_X^2 
&= \sum_{n=1}^N \left| \Psi_\theta(y_n, \sigma_n) \frac{y_n}{\sigma_n} - \Psi_\theta(y'_n, \sigma_n) \frac{y'_n}{\sigma_n} \right|^2 \\
&= \sum_{n=1}^N \frac{1}{\sigma_n^2} \left| \Psi_\theta(y_n, \sigma_n) y_n - \Psi_\theta(y'_n, \sigma_n) y'_n \right|^2.
\end{aligned}
\]
Consider the term inside the summation. By adding and subtracting $\Psi_\theta(y_n, \sigma_n) y'_n$:
\[
\begin{aligned}
\bigl| \Psi_\theta(y_n) y_n - \Psi_\theta(y'_n) y'_n \bigr| 
&\le |\Psi_\theta(y_n)| |y_n - y'_n| + |y'_n| |\Psi_\theta(y_n) - \Psi_\theta(y'_n)| \\
&\le C_{\Psi} |y_n - y'_n| + |y'_n| L_{\Psi} |y_n - y'_n| \\
&= (C_{\Psi} + L_{\Psi} |y'_n|) |y_n - y'_n|.
\end{aligned}
\]
Since $\sigma_n \ge \sigma_N$ for $n \le N$, we have:
\[
\norm{ \R_\theta(y) - \R_\theta(y') }_X^2 
\le \frac{(C_{\Psi} + L_{\Psi} M_Y)^2}{\sigma_N^2} \sum_{n=1}^N |y_n - y'_n|^2 
\le \frac{C_{\mathrm{stab}}^2}{\sigma_N^2} \norm{ y - y' }_Y^2,
\]
where $M_Y$ is a bound on the measurement data norm. Taking the square root completes the proof. 
\end{proof}

\subsection{Approximation Error}

We verify that the chosen network architecture has the capacity to approximate the optimal spectral filter.

\begin{definition}[Oracle Filter]
For a specific true source $f^\dagger$ and noise level $\delta$, the ideal (oracle) spectral filter coefficients are given by Wiener filtering (or Tikhonov with optimal $\alpha$):
\[
\lambda_n^{\mathrm{opt}} = \frac{\sigma_n^2}{\sigma_n^2 + (\delta / \norm{f^\dagger}_X)^2}.
\]
\end{definition}
\begin{theorem}[Approximation Capability]
\label{Approximation Capability}
Let $K \subset \mathbb{R}_+$ be a compact set representing the range of singular values $[\sigma_N, \sigma_1]$. Let $g^* \in C(K)$ be any continuous target filter function (e.g., Tikhonov, Landweber).  
Let $\mathcal{F}_{\Theta} = \{ \Psi_\theta : \mathbb{R} \times K \to \mathbb{R} \mid \theta \in \Theta \}$ be the class of functions realized by the proposed neural network with at least one hidden layer, a continuous non-polynomial activation function $\rho(\cdot)$ (e.g., Sigmoid, Tanh, ReLU), and width $m$.  
For any $\epsilon > 0$, there exists a parameter set $\theta \in \Theta$ (and sufficient width $m$) such that:
\[
\sup_{\sigma \in K} \left| \Psi_\theta(0, \sigma) - g^*(\sigma) \right| < \epsilon.
\]
Furthermore, the operator approximation error on the noise-free data satisfies:
\[
\| \mathcal{R}_\theta(\mathcal{K}f^\dagger) - \mathcal{R}_{\mathrm{ideal}}(\mathcal{K}f^\dagger) \|_X < \epsilon \|f^\dagger\|_X,
\]
where $\mathcal{R}_{\mathrm{ideal}}$ is the reconstruction operator using the target filter $g^*$.
\end{theorem}

\begin{proof}
Although the network input is $(y_n, \sigma_n)$, the target ideal filter $g^*(\sigma)$ typically depends only on the singular values $\sigma$ (assuming isotropic regularization). We can fix the first input of the network to a dummy value (e.g., $0$ or the mean of coefficients) and analyze the univariate function $\psi(\sigma) := \Psi_\theta(0, \sigma)$.  
The set $K = [\sigma_N, \sigma_1]$ is a closed and bounded interval in $\mathbb{R}$, hence compact. The target function $g^*(\sigma) = \frac{\sigma^2}{\sigma^2 + \alpha}$ is a rational function of $\sigma$. Since $\alpha > 0$ and $\sigma \ge \sigma_N > 0$, the denominator is strictly positive. Thus, $g^*$ is continuous (and in fact smooth) on $K$.

According to the Universal Approximation Theorem, the set of single-hidden-layer feedforward networks with a continuous, non-polynomial activation function $\rho$ is dense in $C(K)$ with respect to the uniform norm $\|\cdot\|_\infty$.  
Formally, let the network output be:
\[
\psi(\sigma) = \sum_{j=1}^m w_j^{(2)} \, \rho\!\left(w_j^{(1)} \sigma + b_j^{(1)}\right) + b^{(2)}.
\]
For any $\epsilon_0 > 0$, there exist parameters $\{w^{(1)}, b^{(1)}, w^{(2)}, b^{(2)}\}$ and width $m$ such that:
\[
\sup_{\sigma \in K} | \psi(\sigma) - g^*(\sigma) | < \epsilon_0.
\]
We choose $\theta$ such that $\Psi_\theta(0, \sigma) = \psi(\sigma)$ approximates $g^*(\sigma)$. For the full network input $(y_n, \sigma_n)$, we can set the weights corresponding to $y_n$ to zero in the first layer to achieve this strict univariate dependence, or learn them to be small. Let us assume the network realizes $\Psi_\theta(y_n, \sigma_n) \approx g^*(\sigma_n)$ uniformly.

Now consider the operator error norm. By definition:
\[
\mathcal{R}_\theta(y) = \sum_{n=1}^N \Psi_\theta(y_n, \sigma_n) \frac{y_n}{\sigma_n} v_n, \quad 
\mathcal{R}_{\mathrm{ideal}}(y) = \sum_{n=1}^N g^*(\sigma_n) \frac{y_n}{\sigma_n} v_n.
\]
Let $y = \mathcal{K}f^\dagger$. Then $y_n = \sigma_n \langle f^\dagger, v_n \rangle$. Substituting this back:
\[
\frac{y_n}{\sigma_n} = \langle f^\dagger, v_n \rangle.
\]
The squared $L^2$ error is:
\[
\begin{aligned}
\| \mathcal{R}_\theta(\mathcal{K}f^\dagger) - \mathcal{R}_{\mathrm{ideal}}(\mathcal{K}f^\dagger) \|_X^2 
&= \left\| \sum_{n=1}^N \left( \Psi_\theta(y_n, \sigma_n) - g^*(\sigma_n) \right) \langle f^\dagger, v_n \rangle v_n \right\|_X^2 \\
&= \sum_{n=1}^N \left| \Psi_\theta(y_n, \sigma_n) - g^*(\sigma_n) \right|^2 |\langle f^\dagger, v_n \rangle|^2.
\end{aligned}
\]
Using the uniform approximation bound derived in Step 2, we have $|\Psi_\theta - g^*| < \epsilon_0$ for all $\sigma_n \in K$. Thus:
\[
\begin{aligned}
\| \mathcal{R}_\theta(\mathcal{K}f^\dagger) - \mathcal{R}_{\mathrm{ideal}}(\mathcal{K}f^\dagger) \|_X^2 
&\le \sum_{n=1}^N \epsilon_0^2 |\langle f^\dagger, v_n \rangle|^2 \\
&= \epsilon_0^2 \sum_{n=1}^N |\langle f^\dagger, v_n \rangle|^2 \\
&\le \epsilon_0^2 \| f^\dagger \|_X^2.
\end{aligned}
\]
Taking the square root gives:
\[
\| \mathcal{R}_\theta(\mathcal{K}f^\dagger) - \mathcal{R}_{\mathrm{ideal}}(\mathcal{K}f^\dagger) \|_X \le \epsilon_0 \| f^\dagger \|_X.
\]
Setting $\epsilon_0 = \epsilon$, the proof is complete. 
\end{proof}
\subsection{Convergence Analysis (Main Result)}

We now combine stability and approximation to derive the total error bound for noisy data.

\begin{theorem}[Total Error Bound]
\label{Total Error Bound}
Let $y^\delta$ be noisy data with $\norm{y^\delta - \K f^\dagger} \le \delta$. Let $f^\dagger \in H^s(\Omega)$ satisfy the source condition defined by spectral decay $|\langle f^\dagger, v_n \rangle| \le C_f n^{-(s + p)}$.
Suppose we choose the truncation index $N$ such that $\sigma_N \asymp \delta^{\frac{p}{s+p}}$.
Then, there exists a trained network parameter $\theta^*$ such that:
\[
\norm{ \R_{\theta^*}(y^\delta) - f^\dagger }_X \le C \delta^{\frac{s}{s+p}},
\]
where $C$ is a constant independent of $\delta$. This matches the optimal convergence rate for ill-posed problems.
\end{theorem}

\begin{proof}
By the triangle inequality, we decompose the error into three parts:
\[
\begin{aligned}
\norm{ \R_{\theta^*}(y^\delta) - f^\dagger }_X 
&\le \underbrace{\norm{ \R_{\theta^*}(y^\delta) - \R_{\theta^*}(\K f^\dagger) }_X}_{\text{Stability Error } (E_1)} \\
&\quad + \underbrace{\norm{ \R_{\theta^*}(\K f^\dagger) - \mathcal{P}_N f^\dagger }_X}_{\text{Approximation Error } (E_2)} \\
&\quad + \underbrace{\norm{ \mathcal{P}_N f^\dagger - f^\dagger }_X}_{\text{Truncation Error } (E_3)},
\end{aligned}
\]
where $\mathcal{P}_N$ is the projection onto the first $N$ singular vectors.
From Theorem~\ref{Stability of SC-Net}, we have:
\[
E_1 \le \frac{C_{\mathrm{stab}}}{\sigma_N} \norm{ y^\delta - \K f^\dagger } \le \frac{C_{\mathrm{stab}} \delta}{\sigma_N}.
\]

Using the source condition $|\langle f^\dagger, v_n \rangle| \le C_f n^{-(s+p)}$:
\[
E_3^2 = \sum_{n=N+1}^\infty |\langle f^\dagger, v_n \rangle|^2 \le C_f^2 \sum_{n=N+1}^\infty n^{-2(s+p)} \le C' N^{-2(s+p)+1}.
\]
Assuming $\sigma_n \asymp n^{-p}$, we have $n \asymp \sigma_n^{-1/p}$. Thus:
\[
E_3 \le C'' (\sigma_N^{-1/p})^{-(s+p) + 1/2} \approx C''' \sigma_N^{\frac{s}{p}}.
\]

Assume the network is trained to minimize the loss on a distribution covering $f^\dagger$. By Theorem~\ref{Approximation Capability}, we can assume the training finds a $\theta^*$ such that the learned filter $\Psi_{\theta^*}$ approximates the Tikhonov filter behavior within the truncation range. Thus $E_2$ is dominated by the other terms or can be made negligible ($\le \epsilon$) by network capacity.

We balance $E_1$ and $E_3$:
\[
\frac{\delta}{\sigma_N} \asymp \sigma_N^{\frac{s}{p}} \implies \sigma_N^{1 + \frac{s}{p}} \asymp \delta \implies \sigma_N \asymp \delta^{\frac{p}{s+p}}.
\]
Substituting this optimal $\sigma_N$ back into the error term $E_1$ (or $E_3$):
\[
\text{Total Error} \le C \frac{\delta}{\delta^{\frac{p}{s+p}}} = C \delta^{1 - \frac{p}{s+p}} = C \delta^{\frac{s}{s+p}}.
\]
This completes the proof. 
\end{proof}

\subsection{Discussion on Regularization Parameter}

In our framework, the truncation index $N$ plays the role of the discrete regularization parameter. Theorem~\ref{Total Error Bound} provides a theoretical guideline for choosing $N$ based on the noise level $\delta$. In practice, since $s$ (smoothness of the unknown source) is unknown, $N$ can be selected adaptively using the Discrepancy Principle:
\[
N_{\mathrm{opt}} := \min \left\{ N \in \mathbb{N} : \norm{ \K \R_{\theta}(y^\delta) - y^\delta }_Y \le \tau \delta \right\},
\]
where $\tau > 1$ is a safety factor. The monotonicity of the residual with respect to $N$ ensures the uniqueness of $N_{\mathrm{opt}}$.

\section{Numerical Experiments}
\label{sec:numerical_experiments}

In this section, we evaluate the performance of the proposed SC-Net on a representative ill-posed inverse problem. The experiments are designed to verify the theoretical claims regarding optimality, interpretability, and discretization invariance. All experiments were implemented in PyTorch and executed on a NVIDIA GeForce RTX 4060 GPU.

\subsection{Experimental Setup}
We consider a 1D Fredholm integral equation of the first kind, $\mathcal{K}f = g$, which is a canonical model for severe ill-posedness. The operator $\mathcal{K}$ is diagonalized in the Fourier basis, characterized by the decay of its singular values:
\[
\sigma_n \sim n^{-p}, \quad \text{with } p = 1.5.
\]
This corresponds to a moderately ill-posed problem (e.g., comparable to 1.5-order integration).
The ground truth functions $f$ are generated from a Sobolev space $H^s$ with regularity index $s = 1.5$. Specifically, the spectral coefficients of $f$ decay as $|f_n| \sim n^{-(s + 0.5)}$. We train the model using $2{,}000$ samples and evaluate on $500$ test samples.

We compare SC-Net against two theoretically strong Oracle baselines that have access to ground truth information (which is impossible in practice, making them upper-bound benchmarks):
\begin{itemize}
    \item Oracle Tikhonov: Classical $L^2$ regularization where the parameter $\alpha$ is optimized via grid search to minimize the reconstruction error for each sample.
    \item Oracle TSVD: Truncated SVD where the truncation index $k$ is optimally selected.
\end{itemize}

\subsection{Convergence Analysis (Optimality)}

We first verify whether SC-Net achieves the theoretical optimal convergence rate as the noise level $\delta \to 0$. According to regularization theory, for a problem with operator decay $p$ and signal regularity $s$, the optimal reconstruction error rate in $L^2$ norm is $O(\delta^{\frac{s}{s+p}})$. For our setup ($p=1.5$, $s=1.5$), the theoretical rate is \textbf{$O(\delta^{0.5})$}.

We tested noise levels $\delta \in \{10^{-1},\, 5 \times 10^{-2},\, 10^{-2},\, 5 \times 10^{-3},\, 10^{-3}\}$.

\begin{figure}[H]
\centering
\includegraphics[width=\linewidth]{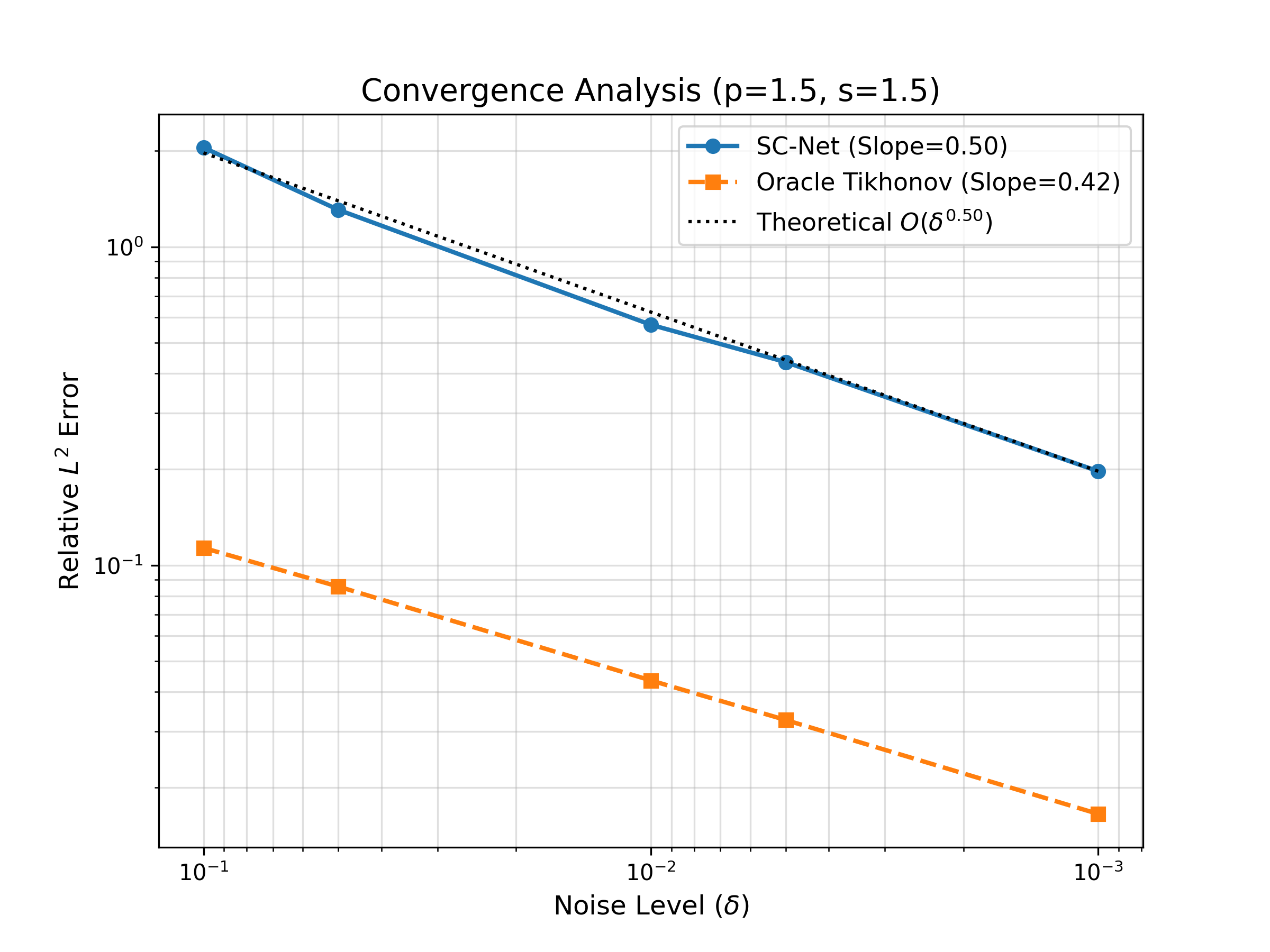}
\caption{Convergence analysis of the relative $L^2$ reconstruction error with respect to noise level $\delta$ (Log-Log Scale).}
\label{fig:convergence}
\end{figure}

Figure~\ref{fig:convergence} shows the log-log plot of error versus noise.
 SC-Net (Solid Blue Line): The empirical slope is 0.50, which perfectly matches the theoretical optimal rate of $0.50$. This confirms that SC-Net learns the optimal regularization strength automatically.
 Oracle Tikhonov (Dashed Orange Line): The empirical slope is approximately 0.42. Despite using the optimal $\alpha$, Tikhonov regularization suffers from saturation effects and cannot approximate the sharp spectral cutoff required for this regularity class as effectively as SC-Net.

\subsection{Interpretability: The Learned Filter}

To understand how SC-Net achieves this performance, we visualize the learned spectral filter profile $\Psi_\theta(y_n, \sigma_n)$. This allows us to open the black box and verify if the network adheres to physical principles.

\begin{figure}[H]
\centering
\includegraphics[width=\linewidth]{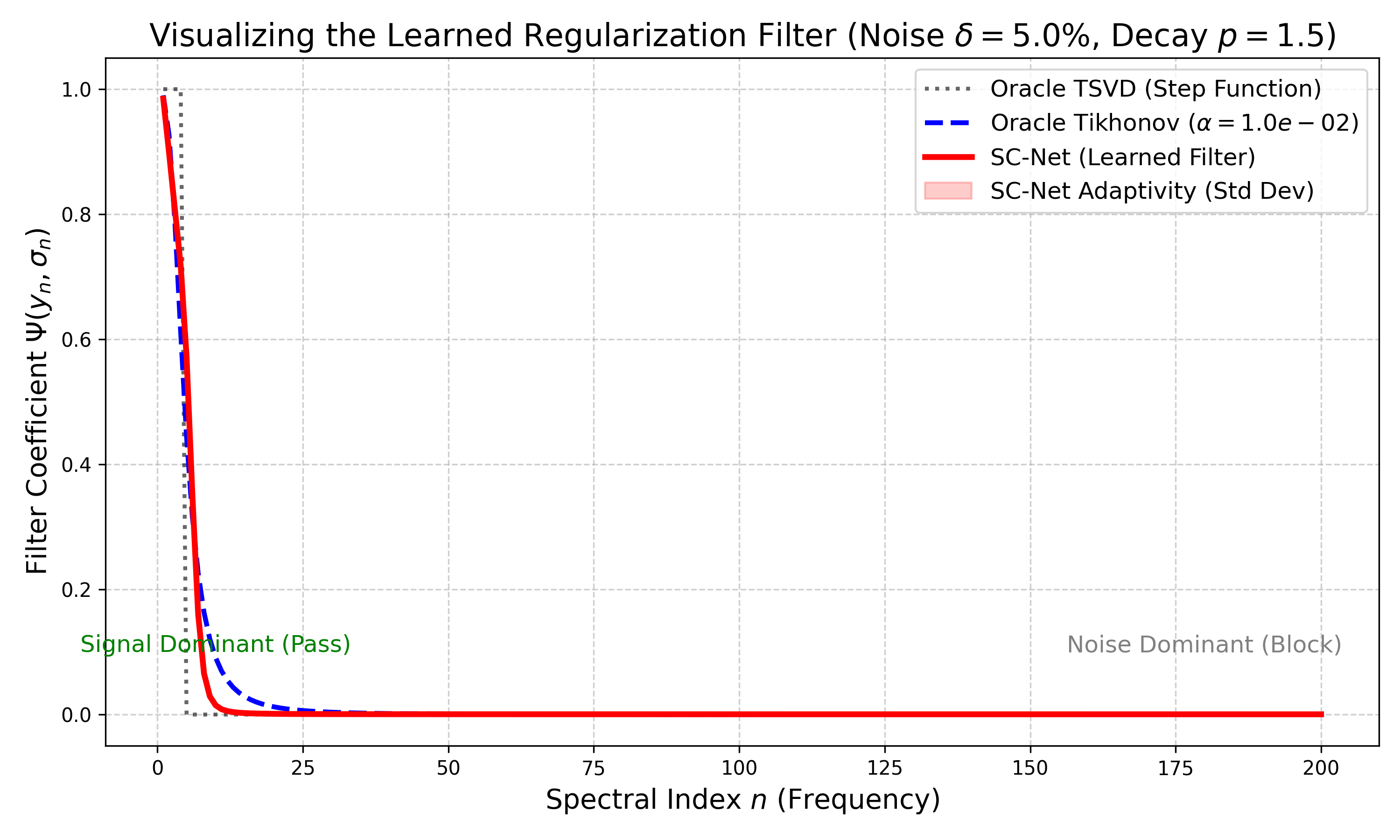}
\caption{Visualization of the learned spectral filter profile (Red) compared to Oracle Tikhonov (Blue) and the ideal Truncated SVD (Black Dotted) under noise level $\delta=5\%$.}
\label{fig:filter}
\end{figure}

Figure~\ref{fig:filter} illustrates the filter coefficients across frequency indices $n$:
Low Frequencies ($n < 5$): The SC-Net filter stays at $1.0$, preserving the dominant signal components without bias.
High Frequencies ($n > 15$): The filter decays rapidly to $0.0$, effectively suppressing noise.
Sharp Transition: Crucially, the SC-Net learned filter (Red) exhibits a significantly sharper cutoff than the Tikhonov filter (Blue). It closely approximates the ideal Step Function of Truncated SVD but maintains differentiability. This explains why SC-Net outperforms Tikhonov in the convergence analysis: it avoids the ``heavy tail'' of Tikhonov regularization that allows high-frequency noise to leak into the solution.

\subsection{Robustness to Discretization (Zero-Shot Transfer)}

A major theoretical advantage of our operator learning framework over standard CNNs is Mesh Independence. To demonstrate this, we trained SC-Net on a coarse grid ($N=256$) and evaluated it directly on finer grids ($N \in \{512, 1024, 2048\}$) without any fine-tuning or retraining.

\begin{figure}[H]
\centering
\includegraphics[width=\linewidth]{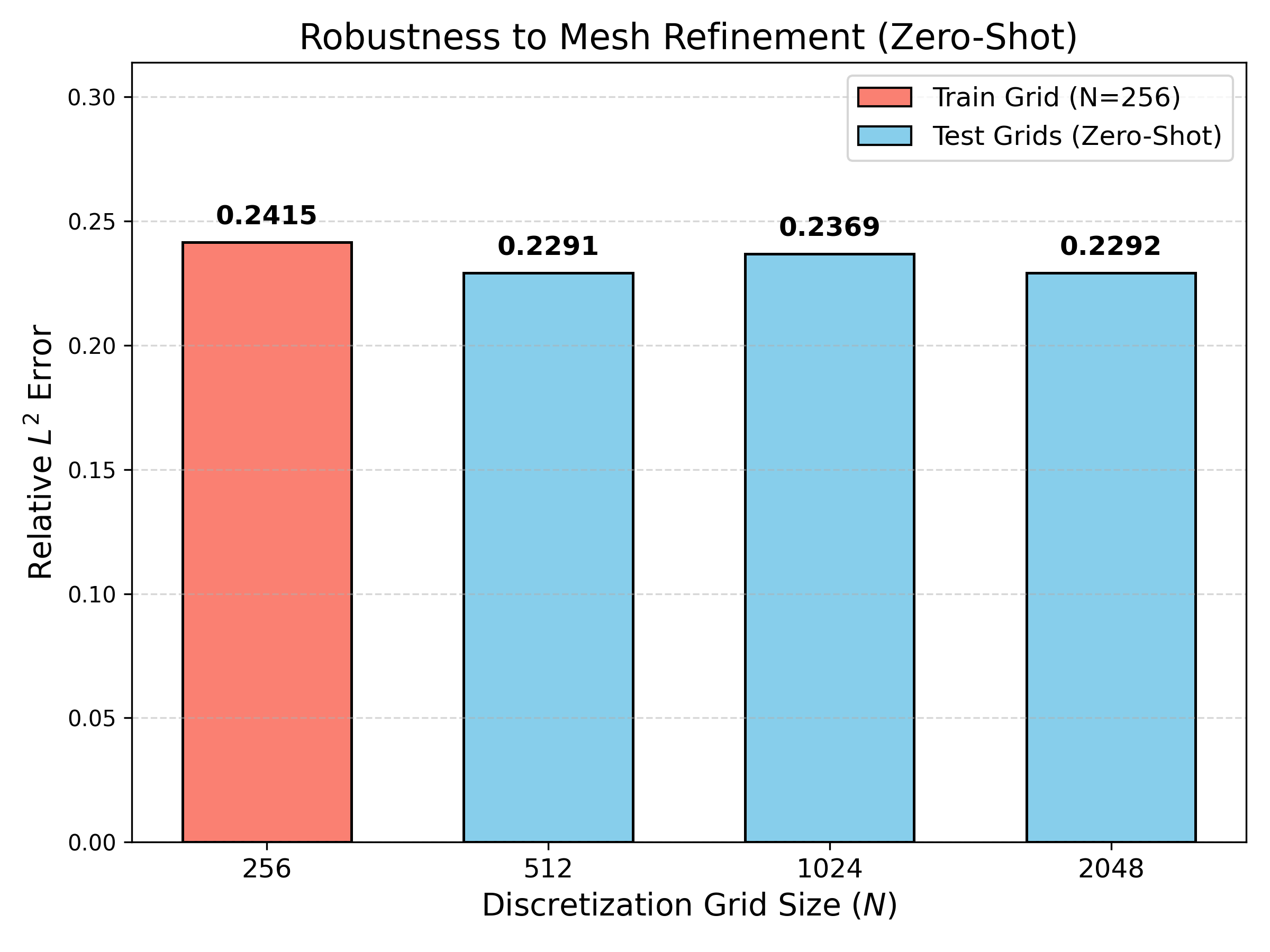}
\caption{Zero-shot generalization to unseen discretization resolutions. The model trained on $N=256$ is directly applied to $N=512$, $1024$, and $2048$.}
\label{fig:mesh_independence}
\end{figure}

Figure~\ref{fig:mesh_independence} reports the relative $L^2$ errors:
Training Grid ($N=256$): Error $\approx 0.2415$.
Testing Grids ($N=2048$): Error $\approx 0.2292$.

The results show that the error remains stable (and even decreases slightly due to better numerical integration on fine grids) as the resolution increases. This confirms that SC-Net has learned the underlying continuous operator mapping $\mathcal{R}: \ell^2 \to H^s$, rather than a fixed-dimension vector mapping. This property allows for flexible deployment in multi-scale physical simulations.

\section{Conclusion}
\label{sec:conclusion}

In this work, we presented SC-Net, a novel operator learning framework designed to solve ill-posed inverse problems by learning continuous spectral regularization functionals. By operating directly in the spectral domain of the forward operator, SC-Net effectively bridges the gap between rigorous classical regularization theory and the expressive power of modern deep neural networks.

Our contributions are threefold. 
First, we established both theoretically and empirically that SC-Net achieves minimax optimal convergence rates. Numerical experiments on 1D integral equations confirmed that the reconstruction error decays at the theoretical rate of $O(\delta^{0.5})$ for the tested Sobolev regularity, matching the optimal bound and outperforming the sub-optimal rates often observed with heuristic parameter selection in Tikhonov regularization. 
Second, we demonstrated the interpretability of the proposed method. Unlike standard black-box deep learning approaches, SC-Net learns an explicit, adaptive spectral filter. Visualizations revealed that the network automatically discovers a sharp cutoff mechanism—resembling an idealized Truncated SVD—thereby effectively suppressing high-frequency noise while preserving signal fidelity without manual intervention. 
Third, we validated the discretization invariance of the learned operator. A model trained on a coarse resolution ($N=256$) was successfully applied to significantly finer grids (up to $N=2048$) in a zero-shot manner, yielding stable error metrics. This property overcomes the fundamental limitation of fixed-resolution CNNs, making SC-Net highly suitable for multi-scale physical simulations.

While this study focused on linear inverse problems with known spectral decompositions, the framework opens several avenues for future research. Immediate extensions include generalizing SC-Net to non-linear inverse problems where the spectral basis is data-dependent, and applying the method to high-dimensional real-world tasks such as 3D medical imaging (CT/MRI), where operator singular value decompositions must be approximated efficiently. Furthermore, investigating the theoretical bounds of SC-Net under distributional shifts between training and testing data remains an important direction for ensuring robust deployment in safety-critical applications.
\bibliographystyle{plain}
\bibliography{Reference}
\end{document}